\documentclass[english,a4paper,12pt]{article}
\RequirePackage{amsmath}
\usepackage{graphicx}

\usepackage{amsxtra, amsfonts, amssymb, amstext, mathtools}%, , amsmath
\usepackage{amsthm}
\usepackage{booktabs}
\usepackage{nicefrac}
\usepackage{xspace}
\usepackage{url}
\usepackage{graphics,color}
\usepackage[algo2e,ruled,vlined,linesnumbered]{algorithm2e}
\usepackage{wrapfig}
\usepackage[caption=false]{subfig}
\usepackage{cite}

\newtheorem{theorem}{Theorem}
\newtheorem{lemma}[theorem]{Lemma}
\newtheorem{corollary}[theorem]{Corollary}

\newcommand{\oea}{\mbox{${(1 + 1)}$~EA}\xspace}

\newcommand{\ooea}{\oea}

\newcommand{\sdfea}{\mbox{SD-FEA$_{\beta,\gamma,R}$}\xspace}
\newcommand{\sdooea}{SD-(1+1)~EA\xspace}

\newcommand{\onemax}{\textsc{OneMax}\xspace}
\newcommand{\LO}{\textsc{Leading\-Ones}\xspace}
\newcommand{\leadingones}{\LO}

\newcommand{\jump}{\textsc{Jump}\xspace}

\newcommand{\R}{\ensuremath{\mathbb{R}}}

\newcommand{\N}{\ensuremath{\mathbb{N}}} % ohne Null!!!

\newcommand{\eps}{\varepsilon}

\let\originalleft\left
\let\originalright\right
\renewcommand{\left}{\mathopen{}\mathclose\bgroup\originalleft}
\renewcommand{\right}{\aftergroup\egroup\originalright}

\usepackage{hyperref}

\DeclarePairedDelimiter\floor{\lfloor}{\rfloor}
\newcommand{\ie}{i.\,e.\xspace}

\DeclareMathOperator{\im}{Im}
\newcommand{\oofea}{(1+1)~FEA$_\beta$\xspace}
\newcommand{\card}[1]{\lvert #1\rvert}
\DeclareMathOperator{\fitnesslevelgap}{FitnessLevelGap}
\DeclareMathOperator{\individualgap}{IndividualGap}
\DeclareMathOperator{\pow}{pow}
\DeclarePairedDelimiter\abs{\lvert}{\rvert}%

\newcommand\recForR{e^{1/\gamma}}
\newcommand\ellr{(1-\gamma)^{-1}\binom{n}{r}\ln(R)}

\newcommand{\sdrlss}{SD-RLS$^{\text{r}}$\xspace}
\newcommand{\sdrlsss}{SD-RLS$^{\text{m}}$\xspace}

\SetKw{Continue}{continue}

\begin{document}
{\sloppy
\title{Stagnation Detection Meets Fast Mutation}

\author{Benjamin Doerr\setcounter{footnote}{0}\thanks{Laboratoire d'Informatique (LIX), CNRS, \'Ecole Polytechnique, Institut Polytechnique de Paris, Palaiseau, France}
\and 
  Amirhossein Rajabi\setcounter{footnote}{1}\thanks{ Technical University of Denmark,
	Kgs. Lyngby, Denmark}
}

\maketitle

\begin{abstract}
Two mechanisms have recently been proposed that can significantly speed up finding distant improving solutions via mutation, namely using a random mutation rate drawn from a heavy-tailed distribution (``fast mutation'', Doerr et al. (2017)) and increasing the mutation strength based on a stagnation detection mechanism (Rajabi and Witt (2020)). Whereas the latter can obtain the asymptotically best probability of finding a single desired solution in a given distance, the former is more robust and  performs much better when many improving solutions in some distance exist. 

In this work, we propose a mutation strategy that combines ideas of both mechanisms. We show that it can also obtain the best possible probability of finding a single distant solution. However, when several improving solutions exist, it can outperform both the stagnation-detection approach and fast mutation. The new operator is more than an interleaving of the two previous mechanisms and it outperforms any such interleaving.
\end{abstract}

\section{Introduction}

Leaving local optima is a challenge for evolutionary algorithms. Mutation-based approaches are challenged by the fact that the typical mutation rate of $p=1/n$ rarely leads to offspring in a larger distance from the parent. When using larger mutation rates, the choice of the mutation rate is critical and small constant-factor deviations from the optimal rate can lead to huge performance losses~\cite[Cor.~4.2]{DoerrLMN17}. 

Two ways to overcome this problem were proposed recently, namely the use of a random mutation rate sampled from a power-law distribution (``fast mutation'')~\cite{DoerrLMN17} and the successive increase of the mutation rate when a stagnation-detection mechanism indicates that the current rate is unlikely to generate solutions not seen yet~\cite{RajabiW20}. An improved version of this stagnation-detection approach~\cite{RajabiW21evocop}, the so-called SD-RLS algorithm based on $k$-bit mutation instead of standard bit mutation, can find a single improving solution in distance $m$ in expected time $(1+o(1)) \binom{n}{m}$ (without knowing that the distance to the desired solution is $m$). Apart from lower order terms, this is the same runtime that can be obtained via a repeated use of the best unbiased mutation operator that is aware of $m$ (which is, naturally, flipping $m$ random bits). It is faster than the fast \oea by a factor of $\Omega(m)$.

While the SD-RLS algorithm thus is very efficient in finding a single desired solution (and thus has very good runtimes on the classic jump functions benchmark (see Section~\ref{sec:jump} for a definition)), this algorithm has a poor performance when there are several improving solutions in distance $m$ as now the stagnation detection approach leads to too much time spent on too small mutation strengths. Taking as an extreme example the generalized jump function~\cite{BamburyBD21} (see again Section~\ref{sec:jump} for a definition) having a valley of low fitness of width $\delta$, $\delta \ge 2$ a constant, in distance $n/4$ from the optimum, we easily see that the SD-RLS takes an expected time of $\Omega(n^{\delta-1})$ to traverse the fitness valley, whereas the \oea both with the classic mutation operator and with fast mutation does so in expected constant time.

\textbf{Our results:} Based on the insight that fast mutation and stagnation detection have complementary strengths, we design a mutation-based approach that takes inspiration from both approaches. We follow, in principle, the basic version of the improved stagnation-detection approach of~\cite{RajabiW21evocop}, that is, we start with mutation strength $r=1$ and increase $r$ gradually. More precisely, when strength~$r$ has been used for a certain number $\ell_r$ of iterations without that an improvement was found, we increase $r$ by one since we assume that no improvement in distance $r$ exists (we omit some technical details in this first presentation of our approach, e.g., that we do not increase $r$ beyond $n/2.1$, and refer the reader to Algorithm~\ref{alg:fastsd} for the full details). Different from~\cite{RajabiW21evocop}, when the current strength is $r$, we do not always flip $r$ random bits as mutation operation, but we choose a random number $X_r$ of bits to flip. This number is equal to $r$ with probability $1-\gamma$, where $\gamma$ is an algorithm parameter that is usually small (a small constant or $o(1)$). With probability $\gamma$, however, $X_r$ deviates from $r$ by an amount following a power-law distribution with exponent $\beta$. The precise definition of this case (see again Algorithm~\ref{alg:fastsd}) is not too important, so for this first exposition we can assume that we sample $D$ from a power-law distribution (with exponent $\beta$) on the positive integers and then, each with probability $1/2$, flip $r+D$ or $r-D$ random bits (where we do nothing if this number is not between $1$ and $n$). 

Since with probability $1-\gamma$ we essentially follow the basic approach of~\cite{RajabiW21evocop}, it is not surprising that we find a single closest improving solution in distance~$m$ in an expected time of $\frac{1}{1-\gamma} (1+o(1)) \binom{n}{m}$, again without that the algorithm needs to know~$m$ (Theorem~\ref{th:escaping-time}). If $\gamma = o(1)$, this is again the optimal time of $(1+o(1)) \binom{n}{m}$ discussed above. We note, however, that our algorithm is simpler than the solution presented in~\cite{RajabiW21evocop}. The basic SD-RLS algorithm proposed in~\cite{RajabiW21evocop} obtains a runtime of $(1+o(1)) \binom{n}{m}$ only with high probability and otherwise fails. To turn this algorithm into one that never fails and has an expected runtime of $(1+o(1)) \binom{n}{m}$, a robust version of the SD-RLS was developed in~\cite{RajabiW21evocop} as well. This version repeats previous phases as follows. When the $\ell_r$ uses of strength~$r$ have not led to an improvement, before increasing the rate to $r+1$, first another~$\ell_i$ iterations are performed with strength $i$, for $i = r-1, \dots, 1$. In our approach, such an additional effort is not necessary since the fast mutations automatically render the algorithm robust.

The use of a heavy-tailed mutation rate also helps in situations where the stagnation-detection mechanism takes too long to use larger mutation strengths. Since in phases $r = 1, \dots, 2m$ the probability to flip $m$ bits is at least $\gamma/2$ times the probability of this event in a run of the fast \oea, it is not surprising that our algorithm finds an improvement in distance~$m$ is at most~$2/\gamma$ times the time of the fast \oea, which as discussed above can be significantly faster than the SD-RLS. Such a result could also have been obtained from  a simple interleaving of SD-RLS and fast \oea iterations. Since our heavy-tailed choices of the mutation strength, however, take into account the current strength~$r$, we often obtain better runtimes, often better than both the SD-RLS and the fast \oea. As the precise statement of these results is technical, we defer the details to Section~\ref{sec:workingprinciples}. As a simple example showing the outperformance of our algorithm, we regard the generalized jump function $\jump_{m,\delta:=m-\Delta}$ for a constant value of $\Delta \ge 2$ and $m = \omega(1)$. This jump function is similar to the classic jump function $\jump_m$, but the valley of low fitness consists not of all search points in positive distance at most $m-1$ from the optimum, but only of those in distance $\Delta+1, \dots, m-1$. Consequently, from the local optimum there is not a single improving solution, but $\Theta(n^\Delta)$. Note that this is still relatively few compared to the fitness valley of size essentially $\binom{n}{m-1}$. On this generalized jump function, the expected runtime of SD-RLS is $O\left(\binom{n}{\delta-1}\ln(R)\right)$, the one of the fast \oea is $O\left(\delta^{\beta-0.5}(en/\delta)^\delta n^{-\Delta}\right)$, and the one of our algorithm is at most~$O\left(\binom{n}{\delta}n^{-\Delta} \gamma^{-1} \right)$ (Corollary~\ref{cor:generaljump}). Since it is also clear that any interleaving of SD-RLS and fast \oea iterations cannot give a better runtime than the one of the two pure algorithms, this result shows that our algorithm can beat SD-RLS and fast EA (and any simple mix of them) when there are several improving solutions in a given distance. 

A short experimental evaluation of the algorithms discussed so far shows that the advantages of our algorithm, proven only via asymptotic runtime results, are also visible for moderate problems sizes.

\textbf{Structure of this paper:} After reviewing the most relevant previous works in Section~\ref{sec:previous}, we introduce our new algorithm in Section~\ref{sec:algo}. In Section~\ref{sec:workingprinciples}, we analyze via mathematical means how our algorithm finds an improvement in distance $m$ both when this is typically achieved in phase $m$ (e.g., when there is only one improving solution in distance $m$) and when this is achieved earlier via the heavy-tailed rates. We use these results in Section~\ref{sec:results} to prove several runtime results, among others, for generalized jump functions. We present some experimental results in Section~\ref{sec:experiments}. In Section~\ref{sec:parameters}, we discuss recommendations on how to set the parameters of our algorithm. We conclude the paper with a short discussion of our results and a pointer to possible future work in Section~\ref{sec:conclusion}.

\section{Previous Works}\label{sec:previous}

This work aims at combining the advantages of stagnation detection and heavy-tailed mutation, so clearly these topics contain the most relevant previous works. Both integrate into the wider questions of how to optimally set the mutation strength of evolutionary algorithms (for this we refer to the recent survey~\cite{DoerrD20bookchapter}) and how evolutionary algorithms can leave local optima (here we refer to~\cite[Section~2.1]{Doerr20gecco} for a discussion of non-elitist approaches and to the introduction of~\cite{DangFKKLOSS18} for a discussion of crossover-based approaches).

For elitist mutation-based approaches, it is clear that when the population has converged to a local optimum the only way to leave this is by mutating a solution from the local optimum into an at least as good solution outside this local optimum. It was observed in~\cite{DoerrLMN17} (the earlier work~\cite{Prugel04} contains similar findings for the special case that the nearest improving solution is in Hamming distance two or three) that standard bit mutation with mutation rate $p=\frac 1n$, which is the most recommended way of doing mutation, is not perfectly suitable to perform larger jumps in the search space. In fact, when the nearest improving solution is in Hamming distance~$m$, then a mutation rate of $p=\frac mn$ is much better, leading to a speed-up by a factor of order $m^{\Theta(m)}$. 

Since~\cite{DoerrLMN17} also observed that missing the optimal rate by a small constant factor leads to performance losses exponential in $m$, it was proposed to use a mutation rate that is drawn from a (heavy-tailed) power-law distribution. Without the need to know $m$, this approach led to runtimes that exceed the ones obtained from the optimal rate $p=\frac mn$ by only a small factor polynomial in $m$. This price for universality can be made as low as $\Theta(m^{0.5+\eps})$, but not smaller than $\Theta(\sqrt m)$. Various variants of heavy-tailed mutation operators have been proposed subsequently, also heavy-tailed choices of other parameters have been used with great success~\cite{FriedrichQW18,FriedrichGQW18,FriedrichGQW18heavysubm,WuQT18,AntipovBD20gecco,AntipovBD20ppsn,AntipovD20ppsn,AntipovBD21gecco,DoerrZ21aaai,CorusOY21, corus2021fast}.

A different way to cope with local optima was proposed in~\cite{RajabiW20}. When an algorithm is stuck in a local optimum for a sufficiently long time, then with high probability it has explored all search points in a certain radius. Consequently, it is safe to increase the mutation rate, which increases the probability to generate more distant solutions. This is the main idea of a series of works on stagnation detection~\cite{RajabiW20,RajabiW21gecco,RajabiW21evocop}. As shown in~\cite{RajabiW20}, this approach can save the polynomial price for universality of the heavy-tailed approach and thus obtain runtimes of the same asymptotic order as when using the optimal (problem-specific) mutation rate. By replacing standard bit mutation with $m$-bit flips, the time to find a particular solution in Hamming distance $m$ was further reduced to $(1+o(1)) \binom{n}{m}$, the same time (apart from lower order terms) one would obtain with the best unbiased mutation operator (which consists of flipping $m$ random bits). 

To be precise, two approaches are discussed in~\cite{RajabiW21evocop}. The simple one, obtained from just replacing standard bit mutation in~\cite{RajabiW20} by $r$-bit mutation, obtains the desired runtimes with high probability, but fails completely with some very small probability. For this reason, also a robust version of the algorithm was proposed in~\cite{RajabiW21evocop}, which by cyclically reverting to smaller mutation strengths overcomes the problem that, with small probability, a given solution in distance $m$ is not found in the phase which uses $m$-bit flips. In~\cite{RajabiW21gecco}, a variation of SD-RLS was proposed that keeps the successful strength after leaving local optima with the help of the radius memory mechanism, which is beneficial on highly multimodal fitness landscapes. The idea of stagnation detection has also been successfully used in multi-objective evolutionary computation~\cite{DoerrZ21aaai}.

\section{Combining Fast Mutation and Stagnation Detection: The Algorithm~\sdfea} \label{sec:algo}

We propose the algorithm~\sdfea for the maximization of pseudo-Boolean functions $f\colon\{0,1\}^n\to \R$ as defined in Algorithm~\ref{alg:fastsd}. The function $\pow(\beta,u)$ samples from a power-law distribution with exponent~$\beta$ and range $[1..u]$ as defined in Equation~\eqref{eq:powerlaw} below.

\begin{algorithm2e}[t!]\label{alg:fastsd}
	\caption{The \sdfea for the maximization of $f\colon\{0,1\}^n\to \R$. Its parameters are the power-law exponent $\beta > 1$, the probability $\gamma$ to deviate from rate $r$ in phase $r$, and the parameter $R$ which defines the maximum length of the $r$-th phase at $\ell_r = (1-\gamma)^{-1} \binom{n}{r} \ln(R)$.}
		Select $x$ uniformly at random from $\{0, 1\}^n$ and set $r_1 \gets 1$\;
		$u\gets 0$\;
		\For{$t \gets 1, 2, \dots$}{
		
		\begin{tabular}{rll}
		Set &$s=r_t$ &with probability~$1\!-\!\gamma$ or\\ &$s=r_t+\pow(\beta, n-r_t)$ &with probability~$\gamma/2$ or\\ 
		&$s=r_t-\pow(\beta, \max\{1,r_t\!-\!1\})$ &with probability~$\gamma/2$\;
		\end{tabular}

        Create $y$ by flipping $s$ bits in a copy of $x$ uniformly at random\;
		$u\gets u+1$\;
		\eIf{$f(y) > f(x)$}{
		$x \gets y$\;
		$r_{t+1}\gets 1$\;
		$u\gets 0$\;}
		{\uIf {$f(y) = f(x)$ \textup{and} $r_t=1$}{
		$x \gets y$\;}
		\eIf%(\tcp*[f]{\footnotesize In this paper, we use $\ell_r = (1-\gamma)^{-1} \binom{n}{r} \ln(R)$.})
		{$u \ge  \ell_{r_t}$}{ 
		$r_{t+1}\gets \min\{r_t+1,\floor{\frac{n}{2.1}}\}$\;
		$u\gets 0$\;
		}{
		$r_{t+1}\gets r_t$\;
		}
		}
		}
	\end{algorithm2e}
	
The general idea of this algorithm is that it increases the mutation strength~$r$ to~$r+1$ when the improvement is not in Hamming distance~$r$ with at least a constant probability (with probability~$1/R$ roughly) using the stagnation detection mechanism. While the strength is~$r$, called in phase~$r$, the algorithm looks at larger or smaller Hamming distances (with probability~$\gamma$) besides using the current strength~$r$. The distribution of the distance of the search radius from to the current strength~$r$ follows a power-law distribution. An integer random variable~$X$ follows a power-law distribution with parameters $\beta$ and $u$ if
\begin{align} \label{eq:powerlaw}
\Pr[X=i]=\begin{cases}
 C_{\beta,u}i^{-\beta} &\text{ if }1\le i\le u,\\
 0                      & \text{ otherwise},
\end{cases}
\end{align}
where $C_{\beta,u}\coloneqq (\sum^{u}_{j=1}j^{-\beta})^{-1}$ is the normalization coefficient. The function~$\pow(\beta,u)$ used in Algorithm~$\ref{alg:fastsd}$ returns a sample from this distribution.

The algorithm starts with a search point selected uniformly at random from the search space~$\{0,1\}^n$ and with the initial strength~$r=1$. There is a counter~$u$ for counting the number of unsuccessful steps in finding a strict improvement with the current strength. When the counter exceeds the maximum phase length~$\ell_r$, the strength~$r$ increases by one but not exceeding~$n/2.1$. When the algorithm makes progress, the counter and strength are reset to their initial values. 

The mutation, which we call $s$-flip in the following, flips exactly~$s$ bits randomly chosen as follows. With probability~$1-\gamma$, the algorithm flips exactly~$r$ bits in phase~$r$. However, with probability~$\gamma$, the algorithm deviates from this choice and instead flips a number of bits which differs from $r$, in either direction, by a value following a power-law distribution. The distribution over~$s$ is analyzed in Lemma~\ref{lem:s-distribution} below.

In this paper, we use maximum phase lengths of
\begin{align} \label{eq:threshold}
\ell_r=\ellr.    
\end{align}
This choice is designed for pseudo-Boolean fitness functions. For other search spaces, the maximum phase length should be $\ell_r=|S_r|/(1-\gamma)\ln (R)$, where~$|S_r|$ is the number of search points in distance~$r$ from the current search point or an upper bound for this.
The maximum phase length defined in Equation~\eqref{eq:threshold} has a parameter~$R$ controlling the probability of failing to find an improvement at the ``right'' strength. To prove our theoretical results, $R$ should be selected at least $\recForR$.
 In Section~\ref{sec:parameters}, we give some recommendations for choosing the parameters of the \sdfea.

As \emph{runtime} of a heuristic algorithm on a fitness function~$f$, we define the first point of time~$t$ where a search point of maximal fitness has been evaluated.
 
\section{Analysis of the \sdfea} \label{sec:workingprinciples}
In this paper, let us define by the \emph{individual gap} of~$x\in\{0,1\}^n$ the minimum Hamming 
distance of~$x$ from points with strictly larger fitness function value, that is,
\[\individualgap(x)\coloneqq \min\{H(x,y):f(y)>f(x) , y\in \{0,1\}^n\}.\]
By the \emph{fitness level} of~$x$, we mean all the search points with fitness value~$f(x)$.
We call the \emph{fitness level gap} of a point $x\in \{0,1\}^n$ 
the maximum of all individual gap sizes in the fitness level of~$x$, \ie,
% \begin{align*}
% \fitnesslevelgap(x) \coloneqq \max_{\{y|f(y)=f(x)\}} \individualgap(y).
% \end{align*}
\begin{align*}
\fitnesslevelgap(x) \coloneqq \max \left\{\individualgap(y): f(y)=f(x), y\in \{0,1\}^n \right\}.
\end{align*}
If the algorithm creates a point at the Hamming distance~$\individualgap(x)$ from the current search point~$x$,  with positive probability an improvement can be found.
Note that $\fitnesslevelgap(x)=1$ is allowed, so the definition also covers search points that are not local optima. As long as a strict improvement is not made, the $\fitnesslevelgap$ remains the same, although the current search point might be replaced with another search point in the fitness level in phase~$1$, that is, when the strength is~1.

% Let us define by the \emph{epoch} of~$i$ the sequence of iterations with the search points in fitness level~$i$. 
We now analyze how the \sdfea finds better selections. Let the current search point be~$x$. We define by phase~$r$ all points of time where radius~$r$ is used for search points with fitness value~$f(x)$, \ie, while in the fitness level of $x$. Let \emph{$E_r$} be the event of \textbf{not} finding the optimum within 
phase~$r$. For $j\ge i$, let $E_i^j$ denote the event of not finding a strict improvement within phases~$i$ to~$j$. Formally, $E_i^j=E_i\cap \dots \cap E_{j}$. 

Before computing the probabilities of these events, we need to know the distribution of the offspring in an iteration. The following lemma will be used throughout this paper, showing the distribution of the number of flipping bits (\ie, the variable~$s$ in Algorithm~\ref{alg:fastsd}) in each iteration. We recall that in phase~$r$, with a relatively large probability~$1-\gamma$, the algorithm flips~$r$ bits. However, with probability~$\gamma$, it uses power-law distributions to flip less or more than $r$ bits.

\begin{lemma} \label{lem:s-distribution}
   Let $r$ be the current strength in an iteration of the algorithm \sdfea. Let $X$ be the integer random variable corresponding to the number of bits that are flipped, that is, the variable~$s$ in Algorithm~\ref{alg:fastsd}. Then
   \begin{align*}
       \Pr[X=\alpha]=\begin{cases}
        (\gamma/2) \cdot C_{\beta,r-1} \cdot (r-\alpha)^{-\beta} & 1\le \alpha < r, \\
        1-\gamma & \alpha = r, \\
        (\gamma/2) \cdot C_{\beta,n-r} \cdot (\alpha-r)^{-\beta} & r<\alpha \le n,
       \end{cases}
   \end{align*}
  and for $r=1$, $\Pr[X=0]=\gamma/2$.
\end{lemma}
\begin{proof}
    It is immediately visible from Algorithm~\ref{alg:fastsd} that $\Pr[X=r]=1-\gamma$.  For $1\le\alpha<r$, we have
    \begin{align*}
        \Pr[X=\alpha] &= \Pr[X< r]\cdot \Pr[X=\alpha \mid X < r] \\
        & = \Pr[X< r]\cdot \Pr[ \pow(\beta,r-1)=r-\alpha] \\
        & = (\gamma/2) \cdot C_{\beta, r-1} (r-\alpha)^{-\beta}.
    \end{align*}
For $\alpha>r$, we similarly obtain
\begin{align*}
        \Pr[X=\alpha] &= \Pr[X> r]\cdot \Pr[X=\alpha \mid X> r] \\
        & = \Pr[X> r]\cdot \Pr[ \pow(\beta,n-r)=\alpha-r] \\
        & = (\gamma/2) \cdot C_{\beta, n-r} (\alpha-r)^{-\beta}.
    \end{align*}
For $r=1$, $\Pr[X=0]=\gamma/2$ because in this case, $\pow(\beta, \max\{1,r_t - 1\})=\pow(\beta, 1)$ returns~$1$ only.
\end{proof}

The following lemma estimates the probability of reaching a phase that is greater than the fitness gap size. In the statement of the lemma, recall that the parameter~$R$ controls the length of the phase.
\begin{lemma} \label{lem:failure-probability}
Let $\beta>1$, $0<\gamma<1$ and $R>1$.
Consider the \sdfea maximizing a pseudo-Boolean fitness function $f\colon\{0,1\}^n\to \R$. 	Let~$x\in\{0,1\}^n$ be the current search point immediately following a strict improvement or the initial search point. Let $m=\individualgap(x)$. Let $E_1^{r-1}$ denote the probability of not finding an improvement in phases~1 to~$r-1$. 
	Then for $m < r \le \floor{\frac{n}{2.1}}$, we have
	\[\Pr[E_1^{r-1}] \le R^{-1-(\gamma/2) \cdot \left(\frac{\ln (1.1)}{\beta}\right)^\beta C_{\beta, n}(r-m-1)}.\]
\end{lemma}

\begin{proof}
Let $p_r$ be a lower bound on the probability of making progress in phase~$r$ in one iteration. Then we have
    \begin{align}
        \Pr[E_1^{r-1}] &\le \Pr[E_m\cap \dots \cap E_{r-1}]
        = \prod^{r-1}_{i=m}\Pr[E_i] \le \prod^{r-1}_{i=m}\left(1-p_i\right)^{\ell_i} \nonumber \\
        &\le \exp\left(-\sum_{i=m}^{r-1}p_i\ell_i\right), \label{eq:E1tor-1}
    \end{align}
where we use the inequality $1+x\le e^x$ for all $x\in \R$.

In the following paragraphs, we aim at bounding $p_i\ell_i$ from below.
For~$i=m$, via Lemma~\ref{lem:s-distribution} and since $\ell_r=\ellr$, we have
\[p_m\ell_m \ge (1-\gamma) \binom{n}{m}^{-1}\cdot (1-\gamma)^{-1}\binom{n}{m} \ln(R) = \ln(R).\]
For~$m<i\le \frac{n}{2.1}$, again using Lemma~\ref{lem:s-distribution}, we have
\[p_i \ge (\gamma/2) C_{\beta, i-1} (i-m)^{-\beta} \binom{n}{m}^{-1},\]
and thus
    \begin{align*}
        p_i\ell_i & \ge (\gamma/2) \cdot C_{\beta, i-1} \frac{\binom{n}{i}\ln(R)}{(1-\gamma)(i-m)^{\beta}\binom{n}{m}} \ge  (\gamma/2) \cdot C_{\beta, n}\frac{\binom{n}{i} \ln(R)}{(i-m)^\beta \binom{n}{m}},
    \end{align*}
where we have used $C_{\beta, n}\le C_{\beta, i-1}$. The last expression is bounded from below by
    \begin{align} \label{eq:pili}
          (\gamma/2) \cdot C_{\beta, n}\frac{ \ln(R)}{(i-m)^\beta}\cdot \frac{\binom{n}{i}}{\binom{n}{i-1}}  \cdots  \frac{\binom{n}{m+1}}{\binom{n}{m}} \ge (\gamma/2) \cdot C_{\beta, n}\frac{\ln(R)}{(i-m)^\beta}(1.1)^{i-m},
    \end{align}
    where we have used $\binom{n}{k}/\binom{n}{k-1}=\frac{n-k+1}{k}\ge 1.1$ for $k\le \floor{\frac{n}{2.1}}$. 
    
    We finally show that $1.1^k/k^\beta\ge \left(\ln (1.1)/\beta\right)^\beta$ for $k\in\N_{\ge1}$. To prove this, let $f(x)=1.1^x/x^\beta$. For $x>0$, its derivative, \ie, $f'(x)$, has only one root, namely $\hat{x}=\frac{\beta}{\ln 1.1 }$. Before and after this point the function is decreasing and increasing, respectively, so $f(\hat{x})$ is the minimum value of the function for $x>0$. We have
    \[f(\hat{x})=\frac{1.1^{\beta/\ln (1.1)}}{(\beta/\ln (1.1))^\beta} \ge \left(\frac{\ln (1.1)}{\beta}\right)^\beta.\]
    Thus, Equation~\eqref{eq:pili} is bounded from below by $(\gamma/2) \cdot C_{\beta, n} \left(\ln (1.1)/\beta\right)^\beta \ln(R)$.
    
{}From Equation~\eqref{eq:E1tor-1}, we obtain
    \begin{align*}
        \Pr[E_1^{r-1}] &\le \exp\left(-\sum_{i=m}^{r-1}p_i\ell_i\right) \le R^{-1-(\gamma/2) \cdot \left(\frac{\ln (1.1)}{\beta}\right)^\beta C_{\beta, n}(r-m-1)}. \qedhere
    \end{align*}
\end{proof}
The next lemma is used to estimate the number of iterations in phases larger than the fitness level gap. With a good choice of the parameters~$\gamma$ and~$R$, the following result becomes $o\left(1/s_m\right)$, that is, the number of steps at larger strengths is negligible compared to the number of steps at the phase~$m$.

\begin{lemma} \label{lem:escaping-time-large-strengths}
Let $\beta>1$, $0<\gamma<1$ and $R\ge \recForR$.
Consider the \sdfea maximizing a pseudo-Boolean fitness function $f\colon\{0,1\}^n\to \R$. 	Let~$x\in\{0,1\}^n$ be the current search point immediately following a strict improvement or the initial search point. Assume $m=\fitnesslevelgap(x)$ and $m\le \floor{n/2.1}$. Let~$s_m$ be a lower bound on the probability that an improvement is found from search points in the fitness level of~$x$ conditional on flipping~$m$ bits.
    Then the expected number of iterations spent with strengths larger than $m$ is at most
    \[O\left(R^{-1}\gamma^{-1} \frac{1}{s_m}\right).\]
\end{lemma}
\begin{proof}
Let $I_r$ be the number of iterations spent in phase~$r$ and $E[I_{>m}]$ denote the expected number of iterations spent with strengths larger than $m$. Then
\[E[I_{>m}] = \sum_{r=m+1}^{\floor{\frac{n}{2.1}}} E[I_r].\]

    With probability~$\Pr[E_1^{r-1}]$, the algorithm does not make progress with strengths less than~$r$. In phase~$r$, the probability of finding an improvement is at least $C_{\beta,r-1}(\gamma/2) (r-m)^{-\beta}\cdot s_m$ in each iteration, by Lemma~\ref{lem:s-distribution}. Thus, for all strengths~$r>m$, using the law of total probability, we have
\begin{align*}
    E[I_{r}] & = \Pr\left[E_1^{r-1}\right] E[I_{r} \mid E_1^{r-1}] + \Pr\left[\overline{E_1^{r-1}}\right] E\left[I_{r} \mid \overline{E_1^{r-1}}\right] \\
    &\le \Pr[E_1^{r-1}] \cdot (C_{\beta,r-1})^{-1}2\gamma^{-1} \cdot \frac{1}{s_m} (r-m)^{\beta}+\Pr\left[\overline{E_1^{r-1}}\right]  \cdot  0 \\
    &= \Pr[E_1^{r-1}] \cdot (C_{\beta,r-1})^{-1}2\gamma^{-1} \cdot \frac{1}{s_m} (r-m)^{\beta}.
\end{align*}
Using Lemma~\ref{lem:failure-probability} and $R\ge\recForR$, we can bound
\begin{align*}
    E[I_{r}] & \le R^{-1-(\gamma/2) \left(\frac{\ln (1.1)}{\beta}\right)^\beta C_{\beta, n}(r-m-1)} (C_{\beta,r-1})^{-1}2\gamma^{-1} \frac{1}{s_m}(r-m)^{\beta} \\
    &= O\left(R^{-1} \gamma^{-1} \frac{1}{s_m} \frac{(r-m)^{\beta}}{\exp\left[(1/2) \cdot \left(\frac{\ln (1.1)}{\beta}\right)^\beta C_{\beta,n}(r-m-1)\right]}\right),
\end{align*}
where we have used $(C_{\beta,r-1})^{-1}=O(1)$ for $\beta>1$.
This results in
\begin{align*}
    \sum_{r=m+1}^{\floor{\frac{n}{2.1}}} E[I_r] & \le O\left(R^{-1}\gamma^{-1} \frac{1}{s_m} \sum_{r=m+1}^{\floor{\frac{n}{2.1}}}\frac{ (r-m)^\beta}{\exp\left[(1/2) \cdot \left(\frac{\ln (1.1)}{\beta}\right)^\beta C_{\beta,n}(r-m)\right]} \right)\\
    & \le O\left(R^{-1}\gamma^{-1}  \frac{1}{s_m}\right),
\end{align*}
where we estimated 
\begin{align*}
\sum_{r=m+1}^{\floor{\frac{n}{2.1}}}\frac{ (r-m)^\beta}{\exp\left[(1/2) \cdot \left(\frac{\ln (1.1)}{\beta}\right)^\beta C_{\beta,n}(r-m)\right]} &=\sum_{r=m+1}^{\floor{\frac{n}{2.1}}}\frac{ (r-m)^\beta}{e^{\Theta(r-m)}} \\
&\le \sum_{k=1}^{\infty}\frac{ k^\beta}{e^{\Theta(k)}} =O(1).
\end{align*}

Therefore, we obtain
\[E[I_{>m}]=\sum_{r=m+1}^{\floor{\frac{n}{2.1}}} E[I_r]=O\left(R^{-1}\gamma^{-1} \frac{1}{s_m}\right)\]
as claimed.
 
\end{proof}

The following lemma, a combinatorial inequality taken from~\cite{RajabiW21evocop}, will be used to count the number of iterations spent with strengths smaller than the fitness level gap.
\begin{lemma}[Lemma~1 in \cite{RajabiW21evocop}] \label{lem:partial-sum} For any integer $m\le n/2$, we have 
	\[\sum_{ i=1}^{m}\binom{n}{i}\leq \frac{n-(m-1)}{n-(2m-1)} \binom{n}{m}.\] 
\end{lemma}

We now present the first main result. In the following theorem, we provide two rigorous upper bounds on the escaping time from a local optimum. 
\begin{theorem} \label{th:escaping-time}
Let $\beta>1$, $0<\gamma<1$ and $R\ge \recForR$.
Consider the \sdfea maximizing a pseudo-Boolean fitness function $f\colon\{0,1\}^n\to \R$. 	Let~$x\in\{0,1\}^n$ be the current search point immediately following a strict improvement or the initial search point. Let $m=\fitnesslevelgap(x)$.
Define $T$ as the time \sdfea takes to create a strict improvement. If $m\le n/2.1$, then
\[E[T]\le  \binom{n}{m} \mathord{\left( \frac{1}{1-\gamma}+ \mathord{O} \left( \frac{m\ln(R)}{(1-\gamma) n}+R^{-1}\gamma^{-1}\right)\right)}.\]
Moreover, for all $m\le n$, we have
\[ E[T]= O\left(2^n\frac{\ln(R)}{1-\gamma} + \gamma^{-1}\binom{n}{m}\abs{\floor{\tfrac{n}{2.1}}-m}^{\beta}\right).\]

\end{theorem}
\begin{proof}
Let $I_r$ be the number of iterations spent in phase~$r$. Using linearity of expectation, we have
\begin{align*}
    E[T] = \sum_{r=1}^{\floor{\frac{n}{2.1}}-1} E[I_r]+E[I_{\floor{\frac{n}{2.1}}}].
\end{align*}

Let first $m\le n/2.1$. For $r<m$, we use that $E[I_r]$ is at most the maximum length of phase~$r$, \ie, $\ell_r=\ellr$. Thus, with Lemma~\ref{lem:partial-sum}, we compute
\begin{align*}
    & \sum_{r=1}^{m-1} E[I_r] \le \sum_{r=1}^{m-1}  \binom{n}{r}
    \frac{\ln (R)}{1-\gamma}\\ % line 1
	& \le  \binom{n}{m-1} \frac{\ln (R)}{1-\gamma} \cdot \frac{n-(m-2)}{n-(2m-3)} \\ % line 2
	 & = \binom{n}{m} \frac{\ln (R)}{1-\gamma} \cdot  \frac{m}{n-m+1} \cdot \frac{n-(m-2)}{n-(2m-3)} .
\end{align*}
Since~$m\le \frac n{2.1}$, the last expression is bounded from above by
	\[
	\sum_{r=1}^{m-1} E[I_r] = O\left(\binom{n}{m}\frac{m\ln (R)}{(1-\gamma)n}\right).\]

 When the strength $r$ equals~$m$, with probability~$1-\gamma$, the algorithm flips exactly~$m$ bits (Lemma~\ref{lem:s-distribution}). When~$m$ bits are flipped, with probability at least~$\binom{n}{m}^{-1}$ an improvement is found. Regarding a truncated geometric distribution with success probability $(1-\gamma)\binom{n}{m}^{-1}$, within at most $(1-\gamma)^{-1}\binom{n}{m}$ iterations in expectation the algorithm finds a better point or the phase is terminated. Thus
	\[E[I_m] \le \frac{\binom{n}{m}}{(1-\gamma)}. \]
For $r>m$, using Lemma~\ref{lem:escaping-time-large-strengths} with $s_m\ge \binom{n}{m}^{-1}$, we obtain
\[E[I_{>m}]=\sum_{r=m+1}^{\floor{\frac{n}{2.1}}} E[I_r]=O\left(R^{-1}\gamma^{-1}\binom{n}{m}\right).\]

Altogether, we have
	\begin{align*}
    E[T] & = \sum_{r=1}^{\floor{\frac{n}{2.1}}} E[I_r]  = \sum_{r=1}^{m-1} E[I_r] + E[I_m] + \sum_{r=m+1}^{\floor{\frac{n}{2.1}}} E[I_r] \\
    & \le \binom{n}{m} \mathord{\left( \frac{1}{1-\gamma}+ \mathord{O} \left( \frac{m\ln(R)}{(1-\gamma) n}+R^{-1}\gamma^{-1}\right)\right)}.
\end{align*}

To prove the second claim, since for $r\le \floor{\frac{n}{2.1}}-1$, we have that $E[I_r]$ is at most the maximum length of phase~$r$, we have

\[E[T] \le \sum_{r=1}^{\floor{\frac{n}{2.1}}-1} \ell_r + E[I_{\floor{\frac{n}{2.1}}}] = \sum_{r=1}^{\floor{\frac{n}{2.1}}-1} \binom{n}{r}(1-\gamma)^{-1}\ln (R) + E[I_{\floor{\frac{n}{2.1}}}].\]

In phase~$\floor{\frac{n}{2.1}}$, the algorithm no longer increases the strength until finding an improvement. Using Lemma~\ref{lem:s-distribution},
the improvement is found with probability at least
\[ \Omega\left((\gamma/2) \cdot \abs{\floor{\tfrac{n}{2.1}}-m}^{-\beta} \cdot\binom{n}{m}^{-1}\right)\]
in each iteration. Using the geometric distribution with this success probability, we obtain
\begin{align*}
E[T] &\le \sum_{r=1}^{\floor{\frac{n}{2.1}}-1} \binom{n}{r}(1-\gamma)^{-1}\ln (R) + O\left(\gamma^{-1}\binom{n}{m}\abs{\floor{\tfrac{n}{2.1}}-m}^{\beta}\right) \\
&= O\left(2^n\frac{\ln(R)}{1-\gamma} + \gamma^{-1}\binom{n}{m}\abs{\floor{\tfrac{n}{2.1}}-m}^{\beta}\right),
\end{align*}
where we have used $\sum_{i=0}^{n}\binom{n}{i} = 2^n$. The second part is proven as desired.
\end{proof}

Theorem~\ref{th:escaping-time} provides a good upper bound on the escaping time from a local optimum when there are only few ways to leave it. However, it is not as good when there are many ways to leave the local optimum. The following theorem considers such scenarios. The constant~$r'$ defined in the theorem basically represents the first phase that the probability of finding one of the improvements is at least constant, and its value is an integer between~1 and~$m$.

\begin{theorem} \label{th:escaping-time-many-solutions}
Let $\beta>1$, $0<\gamma<1$ and $R\ge \recForR$.
Consider the \sdfea maximizing a pseudo-Boolean fitness function $f\colon\{0,1\}^n\to \R$. 	Let~$x\in\{0,1\}^n$ be the current search point immediately following a strict improvement or the initial search point. Let $m=\fitnesslevelgap(x)$ and $s_m$ be a lower bound on the probability that a strict improvement is found from search points in the fitness level of~$x$ conditional on flipping~$m$ bits. Define $T$ as the time \sdfea takes to create a strict improvement.
If~$m\le n/2.1$, then
\[E[T]\le  \frac{1}{s_m} \cdot \frac{1}{\gamma} (m-r')^\beta \cdot O\left( 1+
\frac{r'\ln (R)}{(1-\gamma)n}
 \right),\]
	where $r'=\min\left\{m,\arg\max_{r}\left\{ \binom{n}{r} \le \frac{1}{s_m} \frac{1}{\gamma} (m-r)^\beta \right\}\right\}.$
\end{theorem}
\begin{proof}
Let $I_r$ be the number of iterations spent in phase~$r$. Using linearity of expectation, we have
\begin{align*}
    E[T] = \sum_{r=1}^{\floor{\frac{n}{2.1}}} E[I_r].
\end{align*}

For $r<r'$, we use that $E[I_r]$ is at most  the maximum length of phase~$r$. Thus, by Lemma~\ref{lem:partial-sum}, we have
\begin{align*}
    & \sum_{r=1}^{r'-1} E[I_r] \le \sum_{r=1}^{r'-1}  \binom{n}{r}(1-\gamma)^{-1}
    \ln (R)\\ % line 1
	& \le  \binom{n}{r'-1} \frac{\ln (R)}{(1-\gamma)} \frac{n-(r'-2)}{n-(2r'-3)} \\ % line 2
	 & = \frac{r'}{n-r'+1} \cdot \binom{n}{r'} \frac{\ln (R)}{(1-\gamma)} \frac{n-(r'-2)}{n-(2r'-3)} .
\end{align*}
Since~$r'\le m \le \frac n{2.1}$, the last expression is bounded from above by
	\[ O\left(\binom{n}{r'}\frac{r'\ln (R)}{(1-\gamma)n}\right).\]
Since $\binom{n}{r'} \le \frac{1}{s_m} \frac{1}{\gamma} (m-r')^\beta$ by definition of~$r'$, we estimate
\[\sum_{r=1}^{r'-1} E[I_r] =
O\left(
\frac{1}{s_m} \cdot \frac{1}{\gamma} (m-r')^\beta
\cdot 
\frac{r'\ln (R)}{(1-\gamma)n}
 \right).
\]

In the phases from $r'$ to $m-1$, the probability of finding an improvement is at least~$s_m  (\gamma/2) C_{\beta, n-r'} (m-r')^{-\beta}$, see Lemma~\ref{lem:s-distribution}.
Hence the expected time spent in phases~$r'$ to~$m-1$ is
\[\sum_{r=r'}^{m-1} E[I_r] =O\left(\frac{1}{s_m} \cdot \frac{1}{\gamma} (m-r')^{\beta}\right). \]

 In phase~$m$, where the strength is~$m$,
 exactly~$m$ bits are flipped with probability~$1-\gamma$ (Lemma~\ref{lem:s-distribution}), and in this phase an improvement is found with probability at least~$s_m$ when~$m$ bits are flipped. Thus
	\[E[I_m] \le \frac{1}{s_m} \cdot  \frac{1}{(1-\gamma)}.
	\]
For~$r>m$, using Lemma~\ref{lem:escaping-time-large-strengths} with~$s_m$, we obtain
\[\sum_{r=m+1}^{\floor{\frac{n}{2.1}}} E[I_r] = O\left(R^{-1}\gamma^{-1}s_m^{-1}\right).\]

	Altogether, we have
	\begin{align*}
    &E[T]  = \sum_{r=1}^{\floor{\frac{n}{2.1}}} E[I_r]  \\
    & = \sum_{r=1}^{r'-1} E[I_r] + \sum_{r=r'}^{m-1} E[I_r] + E[I_m] + \sum_{r=m+1}^{\floor{\frac{n}{2.1}}} E[I_r] \\
    & \le 
    O\left(
\frac{1}{s_m} \cdot \frac{1}{\gamma} (m-r')^\beta
\cdot 
\frac{r'\ln (R)}{(1-\gamma)n}
    + \frac{1}{s_m} \cdot \frac{1}{\gamma} (m-r')^{\beta}
    +\frac{1}{s_m(1-\gamma)} + \frac{R^{-1}}{\gamma s_m}\right) \\
    & \le \frac{1}{s_m} \cdot \frac{1}{\gamma} (m-r')^\beta \cdot O\left( 1+
\frac{r'\ln (R)}{(1-\gamma)n}
 \right). \qedhere
\end{align*}  
\end{proof}

After having established some tools for obtaining upper bounds on the time required to escape from local optima, we now analyze the performance of \sdfea on the sub-problems without local optima. 
A maximization function is called \emph{unimodal} in~\cite{DrosteJW02} if and only if there is only one local maximum, where a local maximum is defined as a search point with no better neighbors. In this paper, we use this definition of unimodal functions. Thus, on unimodal functions the gap of all search points in the search space (except for the global optima) is~1, so the algorithm can always make progress in phase~1.

In the following theorem, we state how \sdfea behaves on unimodal functions compared to RLS using an upper bound based on the fitness-level method~\cite{Wegener01}. The theorem and its proof are similar to the second part of Lemma~4 in~\cite{RajabiW21evocop}, and with a good choice of parameters~$\gamma$ and~$R$, the same asymptotic result can be achieved (see the following corollary).

\begin{theorem} \label{th:unimodal}
Let $\beta>1$, $0<\gamma<1$ and $R\ge \recForR$.
	Let $f\colon\{0,1\}^n\to \R$ be a unimodal function and $\card{\im(f)}$ be the number of its fitness values. Let $f_i$ be the $i$-th fitness value in an increasing order of the fitness values of $f$. We consider all fitness levels $A_1, \dots, A_{\card{\im(f)}}$ such that $A_i$ contains search points with fitness value~$f_i$. Let~$s_i$ be a lower bound on the probability that RLS finds an improvement from any search point in~$A_i$. Denote by $T$ the runtime of \sdfea on~$f$. Then 
	\[
	E[T] \le  \left(\frac 1{1-\gamma}+O\left(R^{-1}\gamma^{-1}\right)\right) \sum_{i=1}^{ \card{\im(f)}-1 }\frac 1{s_i}.
	\]
\end{theorem}
\begin{proof}
We define by $I^{(i)}$ the number of all iterations spent to leave the fitness level~$i$. Using linearity of expectation, we have
\[E[T] = \sum^{\card{\im(f)}-1}_{i=1} E[I^{(i)}].\]

	Let $I_r^{(i)}$ be the number of iterations spent in phase~$r$ after a search point for~$A_i$ was found. Then
	\begin{align*}
    I^{(i)}  = \sum^{\floor{\frac{n}{2.1}}}_{r=1} I_r^{(i)}.
\end{align*}

	As long as the strength is~1, the algorithm flips exactly one bit with probability at least~$1-\gamma$ (Lemma~\ref{lem:s-distribution}).
	The worst-case time to leave fitness level~$i$ is at most~$\frac 1{(1-\gamma)s_i}$ using the geometric distribution with success probability $s_i (1-\gamma)$. Hence, for each fitness level~$i$,
	we bound $E[I_1^{(i)}]$ from above 
	by $\frac 1{(1-\gamma)s_i}$, and for $r>1$, we bound $E[I_r^{(i)}]$ from above by using Lemma~\ref{lem:escaping-time-large-strengths} with~$s_m=s_i$. Thus
\[E[I^{(i)}_{>1}]=O\left(R^{-1}\gamma^{-1} \frac{1}{s_i}\right).\]

Altogether, we have
\begin{align*}
E[T] & = \sum^{\card{\im(f)}-1}_{i=1} E[I^{(i)}]  \\
&\le  \sum_{i=1}^{ \card{\im(f)}-1 } \left(\frac 1{s_i(1-\gamma)} + O\left(
R^{-1}\gamma^{-1} \frac{1}{s_i}
\right)\right) \\
& \le \left(\frac 1{1-\gamma}+O\left(R^{-1}\gamma^{-1}\right)\right) \sum_{i=1}^{ \card{\im(f)}-1 }\frac 1{s_i}. \qedhere
\end{align*}
\end{proof}

The following unimodal benchmark functions \onemax and \leadingones have been extensively studied in the literature. They are defined by
\begin{align*}
&\onemax(x)\coloneqq \|x\|_1, \\
&\leadingones(x)\coloneqq \sum_{i=1}^n \prod_{j=1}^i x_j
\end{align*}
for all $x=(x_1,\dots,x_n)\in\{0,1\}^n$, where $\|x\|_1$ is the number of one-bits in the bit string.

The corollary below is a result of Theorem~\ref{th:unimodal} applied on the unimodal functions \onemax with~$s_i=(n-(i-1))/n$ and \leadingones with~$s_i=1/n$.

\begin{corollary}
\label{cor:easyproblems}
	The expected runtime of the \sdfea with $\beta>1$, $\gamma=o(1)$ and $R\ge \recForR$ on \onemax is at most~$(1+o(1))n\ln n$ and on \leadingones is at most~$(1+o(1))n^2$.
\end{corollary}

\section{Analysis on $\jump_{k,\delta}$} \label{sec:results}\label{sec:jump}

In this section, we use the results in the previous section to prove a bound on a generalization of $\jump_\delta$ called $\jump_{k,\delta}$ with two parameters~$k$ and~$\delta$, see Figure~\ref{fig:generaljump} for a depiction. 
\begin{figure}[ht!]
    \centering
    \includegraphics[width=0.75\linewidth]{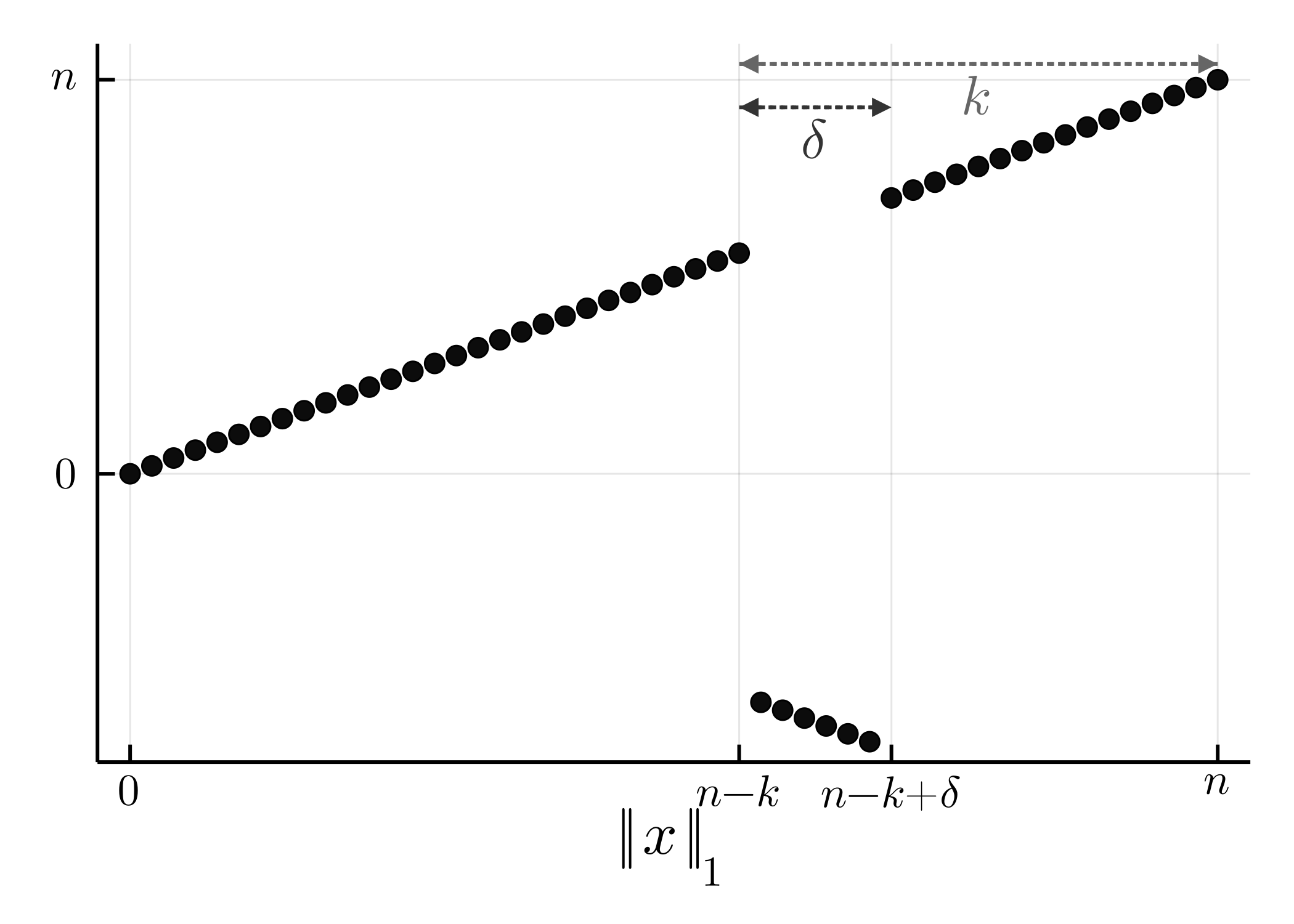}
    \caption{The function $\jump_{k,\delta}$.}
    \label{fig:generaljump}
\end{figure}

This function is based on the well-known \jump benchmark \cite{DrosteJW02}, in which the place of the jump with size~$\delta$ starts at the Hamming distance~$k$ from the global optimum.
In other words, after the jump, there is a unimodal sub-problem of length~$k-\delta$. The classical \jump function is a special case of $\jump_{k,\delta}$ with $k=\delta$, \ie,
$\jump_\delta=\jump_{\delta,\delta}$. 
Formally, for all $x\in \{0,1\}^n$, we have
\[
\jump_{k,\delta}(x) = \begin{cases}
\|x\|_1 & \text{ if } \|x\|_1 \in [0..n-k] \cup [n-k+\delta..n], \\
-\|x\|_1 & \text{otherwise.}
\end{cases}
\]
We refer the interested reader to see~\cite{BamburyBD21} for more information about $\jump_{k,\delta}$, where the performance of the \ooea, the \oofea, and the robust version of SD-RLS (\sdrlss) are carefully analyzed. Also, Rajabi and Witt~\cite{RajabiW21gecco} independently define the jump function with an offset to analyze the recovery time for the strength in the algorithm~SD-RLS with radius memory (\sdrlsss) after leaving the local optimum. Recently, Witt in~\cite{Witt21} analyzes the performance of some other algorithms on the function~$\jump_{k,\delta}$ (which is called $\textsc{JumpOffset}$ in the paper).

We want to show that the algorithm~\sdfea performs relatively efficiently on $\jump_{k,\delta}$ in both cases when $k=\delta$ (\ie, $\jump_\delta$) and $k>\delta$. 
In the first case, when there is only one improving solution, \sdfea with $\gamma=o(1)$ optimizes $\jump_\delta$ as efficient as \sdrlss thanks to Theorem~\ref{th:escaping-time}. The result is formally proven in Theorem~\ref{th:jump}. 

\begin{theorem}
\label{th:jump}
	The expected runtime~$E[T]$ of \sdfea with~$\beta>1$, $\gamma=o(1)$ and $R\ge \recForR$ on $\jump_{\delta}$ with $2\le \delta = o\left(n/\ln(R)\right)$ satisfies 
	\[
	E[T] \le  \binom{n}{\delta}(1+o(1)).
	\]
\end{theorem}
\begin{proof}
	Before reaching a local optimum with~$n-m$ one-bits, $\jump_\delta$ is equivalent to \onemax. Thus,
	 the expected time until \sdfea reaches the local optimum is at most $O(n\ln n)$ via Theorem~\ref{th:unimodal} with~$s_i=(n-(i-1))/n$.

	For a local optimum~$x$ we have $\fitnesslevelgap(x)=\delta$ according to the definition of \jump. Hence, using Theorem~\ref{th:escaping-time}, the algorithm finds the global optimum from the local optimum within the expected time at most
	\[ \binom{n}{\delta}(1+o(1)).\]
	This 
	 dominates the expected time the algorithm spends before reaching the local optimum.  
\end{proof}
For $\gamma=\Theta(1)$, by closely following the analysis of Theorem~\ref{th:jump}, it is easy to see that the expected runtime of \sdfea on $\jump_{\delta}$ is
\[\binom{n}{\delta}\left(\frac{1}{1-\gamma}+o(1)\right),\]
which is still very efficient.

We now present an upper bound on the runtime of the proposed algorithm on $\jump_{k,\delta}$.
\begin{theorem}
\label{th:generaljump}
	The expected runtime~$E[T]$ of \sdfea with $\beta>1$, $0<\gamma<1$ and $R\ge\recForR$ on $\jump_{k,\delta}$ with $\delta = o\left( n/\ln(R) \right)$ satisfies 
	\[
	E[T] = O\left(
		\binom{n}{\delta}\binom{k}{\delta}^{-1} (\delta-r')^\beta \cdot \gamma^{-1} +  n\ln n\right),
	\]
	where $r'=\min\left\{\delta,\arg\max_{r}\left\{ \binom{n}{r} \le \binom{n}{\delta}\binom{k}{\delta}^{-1} \frac{1}{\gamma} (\delta-r)^\beta \right\}\right\}.$
\end{theorem}
\begin{proof}
Until reaching the local optimum with~$n-k$ one-bits, $\jump_{k,\delta}$ is equivalent to \onemax. Thus, the expected time until \sdfea reaches the local optimum is at most $O(n\ln n)$ via Theorem~\ref{th:unimodal} with~$s_i=(n-(i-1))/n$.
	
	For a local optimum~$x$, we have $\fitnesslevelgap(x)=\delta$ according to the definition of~$\jump_{k,\delta}$. Using Theorem~\ref{th:escaping-time-many-solutions} with $s_m=\binom{n}{\delta}^{-1}\binom{k}{\delta}$, the algorithm finds a strict improvement with at least~$n-k+\delta$ one-bits from the local optimum within expected time at most
	\[ O\left(\binom{n}{\delta}\binom{k}{\delta}^{-1} (\delta-r')^\beta \cdot \gamma^{-1}\right),\]
	where we used our assumption $\delta=o(n/\ln (R))$.
	 
	 After leaving the local optimum, $\jump_{k,\delta}$ is again equivalent to \onemax on the second slope. Using the same arguments as in the beginning of the proof, the expected time until \sdfea reaches the global optimum is at most~$O(n\ln n)$ via Theorem~\ref{th:unimodal} with~$s_i=(n-(i-1))/n$.  
\end{proof}

In the following corollary, we see a scenario where we have $r'\ge \delta - c$ for some constant~$c$, resulting in that the term~$(\delta-r')^\beta$ disappears from the asymptotic upper bound. This is also an example where the \sdfea can asymptotically outperform the \oofea.

\begin{corollary}
\label{cor:generaljump}
    Let $\Delta\ge 2$ be a constant.
	The expected runtime~$E[T]$ of \sdfea with $\beta>1$, $0<\gamma<1$ and $R\ge\recForR$ on $\jump_{k,\delta}$ with $k=\omega(1) \cap O(\ln n)$ and $\delta=k-\Delta$  satisfies 
	\[
	E[T] = O\left(
		\binom{n}{\delta}\binom{k}{\delta}^{-1}  \gamma^{-1}\right).
	\]
\end{corollary}

\begin{proof}
We show that $r'$ defined in Theorem~\ref{th:generaljump} is at least~$k-2\Delta$. To this aim, we show that for $r\le k-2\Delta$, we have

\begin{align*}
    \frac{\binom{n}{r}}{\binom{n}{\delta}\binom{k}{\delta}^{-1}\gamma^{-1}(\delta-r)^\beta} \le \gamma \frac{\binom{n}{k-2\Delta}}{\binom{n}{k-\Delta}\binom{k}{\Delta}^{-1}\Delta^\beta} \le \gamma \frac{(en/(k-2\Delta))^{k-2\Delta}(ek/\Delta)^\Delta}{(n/(k-\Delta))^{k-\Delta}\Delta^\beta},
\end{align*}
where we have used $\delta=k-\Delta$ and the inequality~$(n/m)^m \le \binom{n}{m} \le (en/m)^m$.
The last expression equals
\begin{align*}
    \gamma \frac{e^{k-\Delta}k^\Delta}{n^{\Delta}\Delta^{\Delta+\beta}} \frac{(k-\Delta)^{k-\Delta}}{(k-2\Delta)^{k-2\Delta}} &= \gamma \frac{e^{k-\Delta}k^\Delta}{n^{\Delta}\Delta^{\Delta+\beta}} (k-\Delta)^\Delta \left(1+\frac{\Delta}{k-2\Delta} \right)^{k-2\Delta} \\
    & \le \gamma \frac{e^kk^\Delta}{n^{\Delta}\Delta^{\Delta+\beta}} (k-\Delta)^\Delta = o(1),
\end{align*}
where we use the assumption~$k \le \ln n$ and the estimate $1+x\le e^x$ for all $x\in\R$. Thus for $r\le k-2\Delta$ and a large enough $n$, 
we have
\[\binom{n}{r} \le \binom{n}{\delta}\binom{k}{\delta}^{-1}\gamma^{-1}(\delta-r)^\beta,\]
which means that $r'\ge k-2\Delta$.
Therefore, using the result of Theorem~\ref{th:generaljump} with $r'\ge k-2\Delta$, we obtain
\begin{align*}
    E[T] = O\left(
		\binom{n}{\delta}\binom{k}{\delta}^{-1} \Delta^\beta \cdot \gamma^{-1} +  n\ln n\right) = O\left(
		\binom{n}{\delta}\binom{k}{\delta}^{-1} \gamma^{-1} \right),
\end{align*}
where the $O\left(n\ln n\right)$ term is subsumed by the first term according to our assumptions.
\end{proof}

\section{Experiments} \label{sec:experiments}

In this section, we present the results of the experiments carried out to measure the performance of the proposed algorithm and several related ones on concrete problem sizes.

\begin{figure}[t!]
    \centering
	\includegraphics[height=\linewidth]{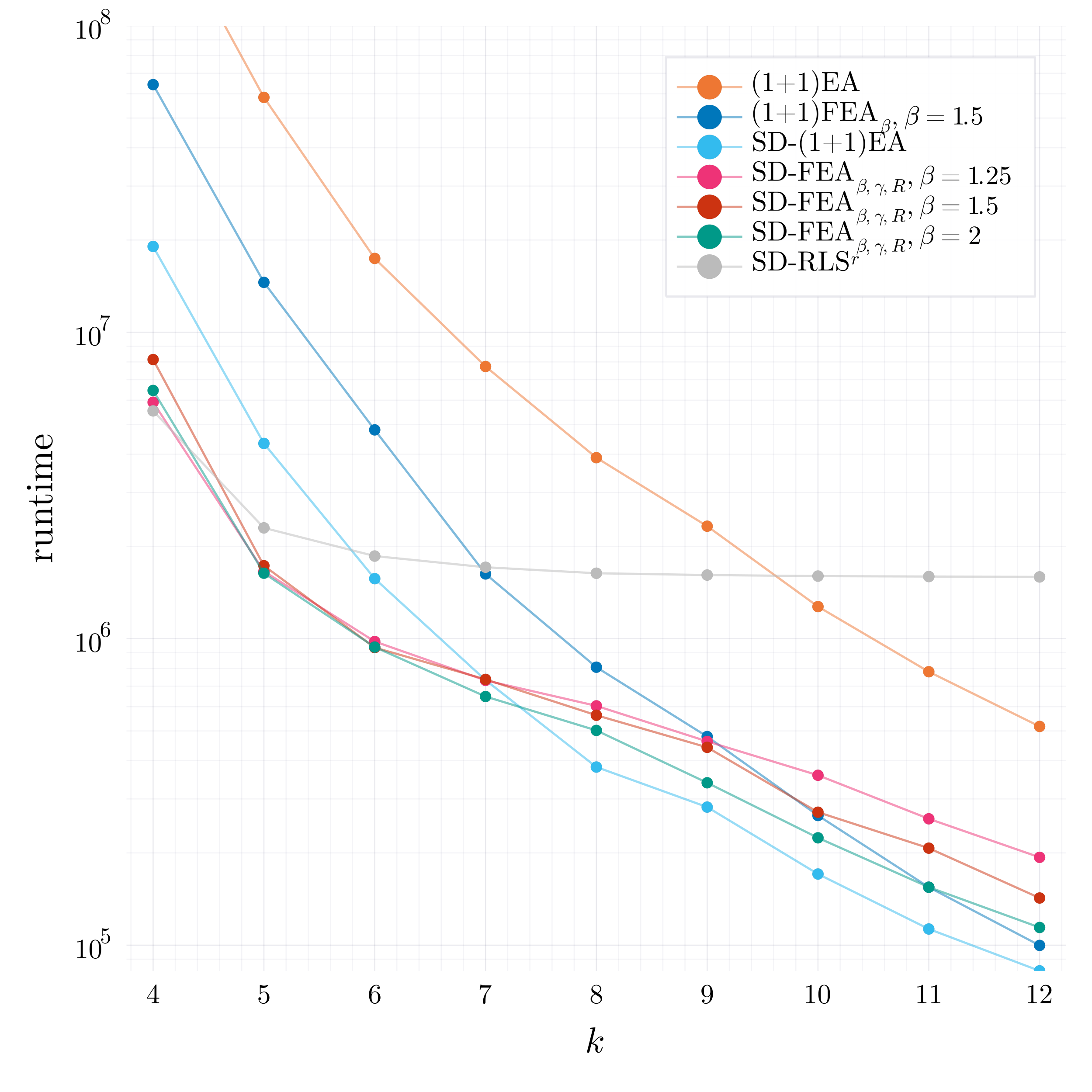}
	\caption{Average number (over 200 runs) of fitness calls the mentioned algorithms spent to optimize $\jump_{k,4}$ ($n=100$) with different values for~$k$.}\label{fig:experimentjump}
\end{figure}

We ran an implementation of \sdfea with $\beta\in \{1.25, 1.5,2 \}$, $\gamma=1/4$ and $R=25$ on the fitness function $\jump_{k,\delta}$ of size $n=100$ with the jump size~$\delta=4$ and $k$ varying from 4 to 13. We recall that we have the classical \jump function for $k=4$.
We compared our algorithm with the classical \ooea with standard mutation rate $1/n$, the \oofea from~\cite{DoerrLMN17} with~$\beta=1.5$, the \sdooea presented in~\cite{RajabiW20} with~$R=n^2$, and \sdrlss from~\cite{RajabiW21evocop} with~$R=n^2$.
The parameter settings for these algorithms were all recommended in the corresponding papers. The parameter values for our algorithm were chosen in an ad-hoc fashion, slightly inspired by our theoretical results.
All data presented is the average number of fitness calls over 200 runs. 
% Considering the hypothesis of identical behavior, we use the Mann-Whitney U-test between all pairs algorithms on all problem configurations., with the result that all p-values are less than~$10^{-2}$.

As can be seen in Figure~\ref{fig:experimentjump}, \sdrlss outperforms the rest of the algorithms for $k=4$, \ie, when there is only one improving solution for local optima. Our \sdfea needs roughly  $(1-\gamma)^{-1}$ times more fitness function calls than that since it ``wastes'' a fraction of $\gamma$ of the iterations on wrong mutation strengths in phase~4. Not all these iterations are wasted as the small differences for different values of $\beta$ show. The higher $\beta$ is, the smaller values the power-law distribution typically takes, meaning that the mutation rate in these iterations stays closer to the ideal one. All three variants of the \sdfea significantly outperform the \oofea, \sdooea and \ooea. As~$k$ is increasing, the average running time of \sdrlss improves only little and remains almost without change after $k=5$; consequently, this algorithm becomes less and less competitive for growing~$k$. This is natural since this algorithm necessarily has to reach phase~4 to be able to flip 4 bits. All other algorithms, especially the \oofea, perform increasingly better with larger~$k$.

In a middle regime of $k \in \{5,6,7\}$, the \sdfea has the best average running time among the algorithms regarded. Although both with $k=4$ and for $k\ge 8$, the \sdfea is not the absolutely best algorithm, but its performance loss over the most efficient algorithm (\sdrlss for $k=4$ and \sdooea for $k\ge 8$) is small. This finding supports our claim that our algorithm is a good approach to leaving local optima of various kinds.

For a large~$k$, such as~10 or~11, the good performance of the \sdooea and \oofea might appear surprising. The reason for the slightly weaker performance of our algorithm is the relatively small width of the valley of low fitness ($\delta = 4$), where our algorithm cannot fully show its advantages, but pays the price of sampling from the right heavy-tailed distribution only with probability~$\gamma/2$.

\section{Recommended Parameters}\label{sec:parameters}

In this section, we use our theoretical and experimental results to derive some recommendations for choosing the parameters $\beta$, $\gamma$, and $R$ of our algorithm. We note that having three parameters for a simple $(1+1)$-type optimizer might look frightening at first, but a closer look reveals that setting these parameters is actually not too critical. 

For the power-law exponent $\beta$, as in~\cite{DoerrLMN17}, there is little indication that the precise value is important. The value $\beta = 1.5$ suggested in~\cite{DoerrLMN17} gives good results even though in our experiments, $\beta=2$ gave slightly better results. We do not have an explanation for this, but in the light of the small differences we do not think that a bigger effort to optimize $\beta$ is justified.

Different from the previous approaches building on stagnation detection, our algorithm also does not need specific values for the parameter $R$, which governs the maximum phase length $\ell_r = \frac{1}{1-\gamma} \binom{n}{r} \ln(R)$ and in particular leads to the property that a single improving solution in distance $m$ is found in phase $m$ with probability $1 - \frac 1R$ (as follows from the proof of Lemma~\ref{lem:failure-probability}). Since we have the heavy-tailed mutations available, it is less critical if an improvement in distance $m$ is missed in phase~$m$. At the same time, since our heavy-tailed mutations also allow to flip more than $r$ bits in phase $r$, longer phases obtained by taking a larger value of $R$ usually do not have  a negative effect on the runtime. For these reasons, the times computed in Theorem~\ref{th:escaping-time} depend very little on~$R$. Since the phase length depends only logarithmically on $R$, we feel that it is safe to choose $R$ as some mildly large constant, say $R=25$.

The most interesting choice is the value for~$\gamma$, which sets the balance between the SD-RLS mode of the algorithm and the heavy-tailed mutations. A large rate $1-\gamma$ of SD-RLS iterations is good to find a single improvement, but can lead to drastic performance losses when there are more improving solutions. Such trade-offs are often to be made in evolutionary computation. For example, the simple RLS heuristic using only 1-bit flips is very efficient on unimodal problems (e.g., has a runtime of $(1+o(1))n\ln n$ on \onemax), but fails on multimodal problems. In contrast, the \oea flips a single bit only with probability approximately $\frac 1e$, and thus optimizes \onemax only in time $(1+o(1)) e n \ln n$, but can deal with local optima. In a similar vein, a larger value for $\gamma$ in our algorithm gives some robustness to situations where in phase $r$ other mutations than $r$-bit flips are profitable -- at the price of a slowdown on problems like classic jump functions, where a single improving solution has to be found. It has to be left to the algorithm user to set this trade-off suitably. Taking the example of RLS and the \oea as example, we would generally recommend a constant factor performance loss to buy robustness, that is, a constant value of $\gamma$ like, e.g., $\gamma = 0.25$. 

\section{Conclusion}\label{sec:conclusion}

In this work, we proposed a way to combine stagnation detection with heavy-tailed mutation. Our theoretical and experimental results indicate that our new algorithm inherits the good properties of the previous stagnation detection approaches, but is superior in the following respects.

\begin{itemize}
\item The additional use of heavy-tailed mutation greatly speeds up leaving a local optimum if there is more than one improving solution in a certain distance~$m$. This is because to leave the local optimum, it is not necessary anymore to complete phase $m-1$.
\item Compared to the robust SD-RLS, which is the fairest point of comparison, our algorithm is significantly simpler, as it avoids the two nested loops (implemented via the parameters $r$ and $s$ in~\cite{RajabiW21evocop}) that organize the reversion to smaller rates. Compared to the SD-\oea, our approach can obtain the better runtimes of the SD-RLS approaches in the case that few improving solutions are available, and compared to the simple SD-RLS of~\cite{RajabiW21evocop}, our approach surely converges.
\item Again comparing our approach to the robust SD-RLS, our approach gives runtimes with exponential tails. Let $m$ be constant. If the robust SD-RLS misses an improvement in distance $m$ in the $m$-th phase and thus in time $O(n^m)$ -- which happens with probability $n^{-\Theta(1)}$ for typical parameter settings --, then strength $m$ is used again only after the $(m+1)$-st phase, that is, after $\Omega(n^{m+1})$ iterations. If our algorithm misses such an improvement in phase $m$, then in each of the subsequent $\ell_{m+1} = \Omega(n^{m+1})$ iterations, it still has a chance of $\Omega\left(n^{-m}\gamma\right)$ to find this particular improvement. Hence the probability that finding this improvement takes $\Omega(n^{m+1})$ time, is only $(1-\Omega\left(n^{-m}\gamma\right))^{\Omega(n^{m+1})} \le \exp(- \Omega(n\gamma))$.
\end{itemize}

As discussed in Section~\ref{sec:parameters}, the three parameters of our approach are not too critical to set. For these reasons, we believe that our combination of stagnation detection and heavy-tailed mutation is a very promising approach. 

As the previous works on stagnation detection, we have only analyzed stagnation detection in the context of a simple hillclimber. This has the advantage that it is clear that the effects revealed in our analysis are truly caused by our stagnation detection approach. Given that there is now quite some work studying stagnation detection in isolation, for future work it would be interesting to see how well stagnation detection (ideally in the combination with heavy-tailed mutation as proposed in this work) can be integrated into more complex evolutionary algorithms. 

\section*{Acknowledgement}
This work was supported by a public grant as part of the
Investissements d'avenir project, reference ANR-11-LABX-0056-LMH,
LabEx LMH and a research grant by the Danish Council for Independent Research (DFF-FNU  8021-00260B) as well as a travel grant from the Otto Mønsted foundation.

\bibliographystyle{alpha}

\bibliography{alles_ea_master,ich_master, references}
% }%end sloppy
% \end{document}

}%end sloppy
\end{document}